\documentclass{article}

\usepackage{microtype}
\usepackage[utf8]{inputenc}

\usepackage{subcaption}
\usepackage{graphicx}
\usepackage{booktabs} 
\usepackage{formattings/hadianeh}
\usepackage{formattings/macros}
\usepackage{hyperref}

\newcommand{\sinreq}[0]{\textbf{\textsf{{WaveQ}}}\xspace}
\usepackage{amsmath}

\usepackage{tikz}
\newcommand{\ballnumber}[1]{\tikz[baseline=(myanchor.base)] \node[circle,fill=.,inner sep=1pt] (myanchor) {\color{-.}\bfseries\footnotesize #1};}

\usepackage{calrsfs}

\DeclareMathAlphabet{\pazocal}{OMS}{zplm}{m}{n}

\usepackage{ amssymb }
\usepackage{ mathrsfs }
\usepackage{mathtools}
\DeclarePairedDelimiter\ceil{\lceil}{\rceil}

\newcommand{\quantize}{\mathrm{quantize}}
\newcommand{\round}{\mathrm{round}}
\newcommand{\step}{\mathrm{step}}

\usepackage{amsmath,amssymb,amsfonts}
\usepackage[font={sf,footnotesize,bf}]{caption}

\usepackage{amsmath}
\usepackage{amsfonts}
\usepackage{amssymb}
\usepackage[english]{babel}
\usepackage{graphicx}
\usepackage{amsthm}
\usepackage{hyperref}
\usepackage{accents}
\usepackage [autostyle, english = american]{csquotes}
\MakeOuterQuote{"}
\numberwithin{equation}{section}
\theoremstyle{plain}
\newtheorem*{theorem*}{Theorem}
\newtheorem*{lemma*}{Lemma}
\newtheorem{theorem}{Theorem}
\newtheorem{lemma}{Lemma}[section]

\newtheorem{question}[lemma]{Question}

\theoremstyle{definition}
\newtheorem{definition}[lemma]{Definition}

\newcommand{\be}{\begin{equation}}
\newcommand{\ee}{\end{equation}}
\newcommand{\ba}{\begin{array}{l}}
\newcommand{\ea}{\end{array}}

\usepackage[accepted]{formattings/icml2019}

\icmltitlerunning{}

\begin{document}

\twocolumn[
\icmltitle{WaveQ: Gradient-Based Deep Quantization of Neural Networks through Sinusoidal Adaptive Regularization}



\icmlsetsymbol{equal}{*}


\begin{icmlauthorlist}
\icmlauthor{Ahmed T. Elthakeb}{to}
\icmlauthor{Prannoy Pilligundla}{goo}
\icmlauthor{Fatemehsadat Mireshghallah}{goo}
\icmlauthor{Tarek Elgindi}{ed}
\icmlauthor{Charles-Alban Deledalle}{bos}
\icmlauthor{Hadi Esmaeilzadeh}{goo} \\
\textbf{A}lternative \textbf{C}omputing \textbf{T}echnologies ({\color[HTML]{0B6121}{\textbf{ACT}}}) Lab\\
University of California San Diego  \\
\end{icmlauthorlist}


\icmlaffiliation{to}{Department of Electrical and Computer Engineering, University of California San Diego.} 
\icmlaffiliation{goo}{Department of Computer Science, University of California San Diego.} 
\icmlaffiliation{ed}{Department of Mathematics, University of California San Diego.} 
\icmlaffiliation{bos}{Institut de Mathèmatiques de Bordeaux, CNRS, Université de Bordeaux, Bordeaux INP} 

\icmlcorrespondingauthor{Ahmed T. Elthakeb}{a1yousse@eng.ucsd.edu}

\icmlkeywords{Machine Learning, ICML}

\vskip 0.1in
]



\printAffiliationsAndNotice{}  

\begin{abstract}
As deep neural networks make their ways into different domains and application, their compute efficiency is becoming a first-order constraint.
Deep quantization, which reduces the bitwidth of the operations (below eight bits), offers a unique opportunity as it can reduce both the storage and compute requirements of the network super-linearly.
%
%
However, if not employed with diligence, this can lead to significant accuracy loss. 
Due to the strong inter-dependence between layers and exhibiting different characteristics across the same network, choosing an optimal bitwidth per layer granularity is not a straight forward.
As such, deep quantization opens a large hyper-parameter space, the exploration of which is a major challenge.
%
%
%
We propose a novel sinusoidal regularization, called \sinreq, for deep quantized training.
%
%
Leveraging the sinusoidal properties, we seek to learn multiple quantization parameterization in conjunction during \emph{gradient-based} training process.
Specifically, we \emph{learn} \emph{{(i) a per-layer quantization bitwidth}} along with \emph{{(ii) a scale factor}} through learning the period of the sinusoidal function.
At the same time, we exploit the periodicity, \emph{differentiability}, and the local convexity profile in sinusoidal functions to automatically propel \emph{{(iii) network weights}} towards values quantized at levels that are jointly determined.
We show how \sinreq balance compute efficiency and accuracy, and provide a heterogeneous bitwidth assignment for quantization of a large variety of deep networks (AlexNet, CIFAR-10, MobileNet, ResNet-18, ResNet-20, SVHN, and VGG-11) that virtually preserves the accuracy.
Furthermore, we carry out experimentation using fixed homogenous bitwidths with 3- to 5-bit assignment and show the versatility of \sinreq in enhancing quantized training algorithms (DoReFa and WRPN) 
with about $4.8\%$ accuracy improvements on average, and then outperforming multiple state-of-the-art techniques.
%
%
\end{abstract}

\vspace{-0.5cm}
\section{Introduction}
\vspace{-0.1cm}
\label{sec:intro}
%
%
%
%
Quantization, in general, and deep quantization (below eight bits), in particular, aim to not only reduce the compute requirements of DNNs but also significantly reduce their memory footprint~\cite{Zhou2016DoReFaNetTL,Judd2016StripesBD,Hubara2017QNN,Mishra2017WRPNWR,bitfusion}.
Nevertheless, without specialized training algorithms, quantization can diminish the accuracy.
%
%
%
As such, the practical utility of quantization hinges upon addressing two fundamental challenges: (1) discovering the appropriate bitwidth of quantization for each layer while considering the accuracy; and (2) learning weights in the quantized domain for a given set of bitwidths.

This paper formulates both of these problems as a \emph{gradient-based} joint optimization problem by introducing in the training loss an additional and novel sinusoidal regularization term, called \sinreq.
The following two main insights drive this work.
(1) Sinusoidal functions ($sin^2$) have inherent periodic minima and by adjusting the period, the minima can be positioned on quantization levels corresponding to a bitwidth at per-layer granularity.
(2) As such, sinusoidal period becomes a direct and continuous representation of the bitwidth.
Therefore, \sinreq incorporates this continuous variable (i.e., period) as a differentiable part of the training loss in the form of a regularizer.
Hence, \sinreq can piggy back on the stochastic gradient descent that trains the neural network to also learn the bitwidth (the period). 
Simultaneously this parametric sinusoidal regularizer pushes the weights to the quantization levels ($sin^2$ minima).

By adding our sinusoidal regularizer to the original training objective function, our method automatically yields the bitwidths for each layer along with nearly quantized weights for those bitwidths.
In fact, the original optimization procedure itself is harnessed for this purpose, which is enabled by the differentiability of the sinusoidal regularization term.
As such, quantized training algorithms~\cite{Zhou2016DoReFaNetTL,Mishra2017WRPNWR} that still use some form of backpropagation~\cite{rumelhart:errorpropnonote} can effectively utilize the proposed mechanism by modifying their loss.
Moreover, the proposed technique is flexible as it enables heterogenous quantization across the layers.
The \sinreq regularization can also be applied for training a model from scratch, or for fine-tuning a pretrained model.

In contrast to the prior inspiring works~\cite{DBLP:journals/corr/abs-1905-11452, DBLP:journals/corr/abs-1902-08153}, \sinreq is the only technique that casts finding the  bitwidthes and the corresponding quantized weights as a simultaneous gradient-based optimization through sinusoidal regularization during the training process.
We also prove a theoretical result providing insights on why the proposed approach leads to solutions preserving the original accuracy while being prone to quantization.
%
We evaluate \sinreq using different bitwidth assignments across different DNNs (AlexNet, CIFAR-10, MobileNet, ResNet-18, ResNet-20, SVHN, and VGG-11).
To show the versatility of \sinreq, it is used with two different quantized training algorithms, DoReFa~\cite{Zhou2016DoReFaNetTL} and WRPN~\cite{Mishra2017WRPNWR}.
%
Over all the bitwidth assignments, the proposed regularization, on average, improves the top-1 accuracy of DoReFa by 4.8\%.
%
%
The reduction in the bitwidth, on average, leads to 77.5\% reduction in the energy consumed during the execution of these networks.

\section{Joint Learning of Layer Bitwidths and Quantized Parameters}
\label{sec:method}
\begin{figure}
  \centering 
  \includegraphics[width=0.5\textwidth]{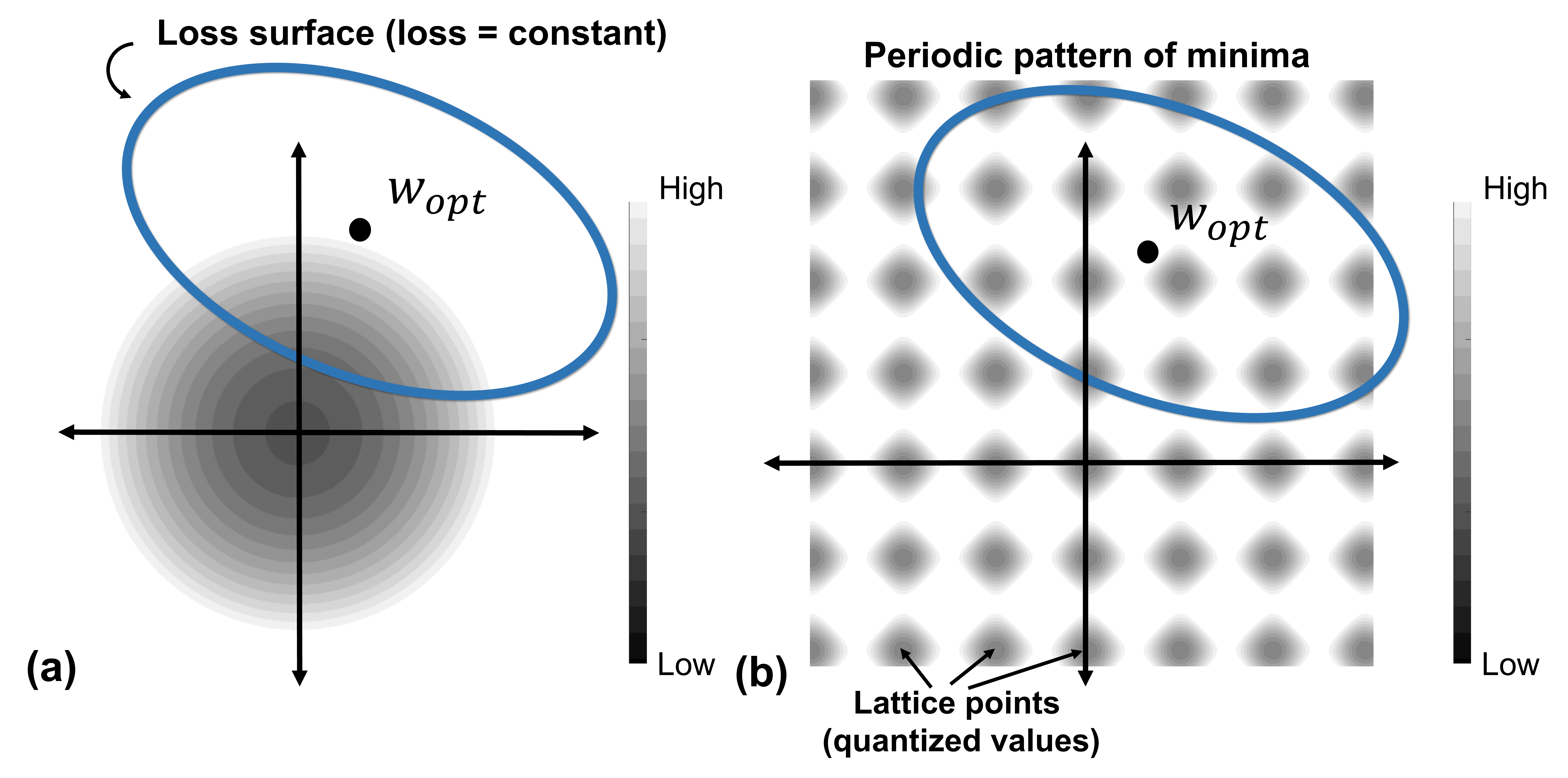}
  \vspace{-0.65cm}
  \caption{Geometrical sketch for a hypothetical loss surface (original task loss to be minimized) and an extra regularization term in 2-D weight space visualizing the induced 2-D gradients for (a) weight decay, and (b) \sinreq respectively. $w_{opt}$ is the optimal solution considering the original task loss alone.}
  \label{fig:reg}
  \vspace{-0.65cm}
\end{figure}
\begin{figure*}
  \centering 
  \includegraphics[width=0.9\textwidth]{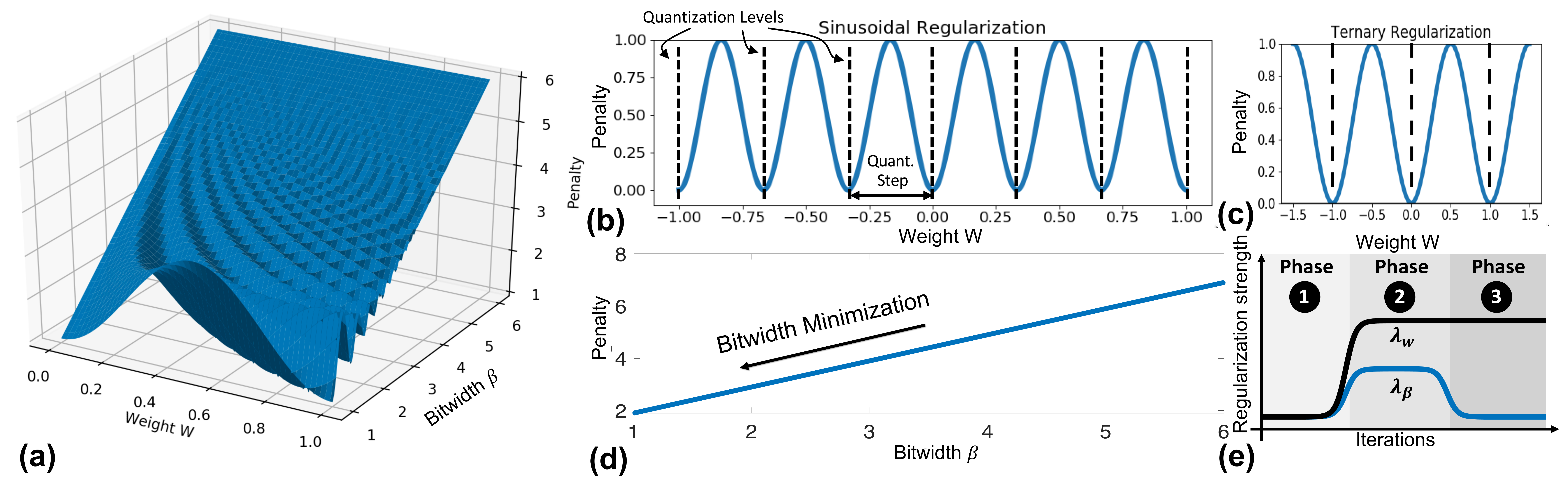}
    \vspace{-0.5cm}
  \caption{(a) 3-D visualization of the proposed generalized objective \sinreq. (b) \sinreq 2-D profile, $w.r.t$ weights, adapting for arbitrary bitwidths, (c) example of adapting to ternary quantization. (d) \sinreq 2-D profile $w.r.t$ bitwidth. (e) Regularization strengths profiles, $\lambda_w$, and $\lambda_\beta$, across training iterations.}
  \label{fig:function}
  \vspace{-0.55cm}
\end{figure*}
Our proposed method \sinreq exploits weight regularization in order to automatically quantize a neural network while training. To that end, Sections~\ref{sec:loss} describes the role of regularization in neural networks and then Section~\ref{sec:sinreq} explains \sinreq in more details.
\subsection{Background}%
\label{sec:loss}
\niparagraph{Loss landscape of neural networks.}
Neural networks' loss landscapes are known to be highly non-convex and generally poorly understood. It has been empirically verified that loss surfaces for large neural networks have many local minima that are essentially equivalent in terms of test error \cite{DBLP:journals/corr/ChoromanskaHMAL14,DBLP:journals/corr/abs-1712-09913}. 
%
%
This opens up the possibility of adding soft constrains as extra custom objectives to optimize during the training process, in addition to the original objective (i.e., to minimize the accuracy loss). 
The added constraint could be with the purpose of increasing generalization performance or imposing some preference on the weights values.
%

%
\niparagraph{Regularization in neural networks.}
Neural networks often suffer from redundancy of parameterization and consequently they commonly tend to overfit.
Regularization is one of the commonly used techniques to enhance generalization performance of neural networks.
Regularization effectively constrains weight parameters by adding a term (regularizer) to the objective function that captures the desired constraint in a soft way. 
This is achieved by imposing some sort of preference on weight updates during the optimization process. 
As a result, regularization seamlessly leads to unconditionally constrained optimization problem instead of explicitly constrained which, in most cases, is much more difficult to solve.
%

%
\niparagraph{Classical regularization: weight decay.}
The most commonly used regularization technique is known as \textit{weight decay}, which aims to reduce the network complexity by limiting the growth of the weights,
see Figure \ref{fig:reg} (a). 
It is realized by adding a regularization term $R$ to the objective function $E$ that penalizes large weight values as follows:
 \begin{equation}
E(w) = E_o(w) + R(w) \quad\text{with}\quad
R(w) = \frac{\lambda}{2} \sum_{i}\sum_{j} {w^2 _{ij}}
 \end{equation}
where $w$ is the collection of all synaptic weights, $E_o$ is the original loss function, and $\lambda$ is a parameter governing how strongly large weights are penalized. The $j$-th synaptic weight in the $i$-th layer of the network is denoted by $w_{ij}$.

%
\subsection{\sinreq Regularization}%
\label{sec:sinreq}
\niparagraph{Proposed objective.}
Here, we propose our sinusoidal based regularizer, \sinreq, which consists of the sum of two terms defined as follows:
 \begin{equation}
  \hspace{-.8em}
  R(w;\beta) = \underbrace{\lambda_{w} \sum_{i}\sum_{j} {\frac {\sin^{2}\left({\pi w_{ij}{(2^{\beta_i}-1)}}\right)}{2^{\beta_i}}}}_\textrm{Weights quantization regularization} \quad+\quad\! \underbrace{\lambda_{\beta} \sum_{i} \beta_i}_{\hspace{-2em}\textrm{Bitwidth regularization}\hspace{-2em}} 
 \label{eq:sinreq}
 \end{equation}
where $\lambda_w$ is the weights quantization regularization strength which governs how strongly weight quantization errors are penalized, and $\lambda_\beta$ is the bitwidth regularization strength.
The parameter $\beta_i$ is proportional to the quantization bitwidth as will be further elaborated on below. 
Figure \ref{fig:function} (a) shows a 3-D visualization of our regularizer, $R$.
Figure \ref{fig:function} (b), (c) show a 2-D profile w.r.t weights ($w$), while (d) shows a 2-D profile w.r.t the bitwidth $(\beta)$.

\niparagraph{Periodic sinusoidal regularization.}
%
%
As shown in Equation~\eqref{eq:sinreq}, the first regularization term is based on a periodic function (sinusoidal) that provides a smooth and differentiable loss to the original objective, Figure \ref{fig:function} (b), (c). 
The periodic regularizer induces a periodic pattern of minima that correspond to the desired quantization levels. Such correspondence is achieved by matching the period to the quantization step ($1/(2^{\beta_i}-1)$) based on a particular number of bits ($\beta_i$) for a given layer $i$.
For the sake of simplicity and clarity, Figure~\ref{fig:reg}(a) and (b) depict a geometrical sketch for a hypothetical loss surface (original objective function to be minimized) and an extra regularization term in 2-D weight space, respectively. 
For weight decay regularization (Figure~\ref{fig:reg} (a)), the faded circle shows that as we get closer to the origin, the regularization loss is minimized. 
The point $w_{opt}$ is the optimum just for the loss function alone and the overall optimum solution is achieved by striking a balance between the original loss term and the regularization loss term.
In a similar vein, Figure \ref{fig:reg}(b) shows a representation of the proposed periodic regularization for a fixed bitwidth $\beta$. 
A periodic pattern of minima pockets are seen surrounding the original optimum point. 
The objective of the optimization problem is to find the best solution that is the closest to one of those minima pockets where weight values are nearly matching the desired quantization levels, hence the name quantization-friendly.

\niparagraph{Quantizer.}
Before delving into how our sinusoidal regularizer is used for quantization, we discuss how quantization works. 
Consider a floating-point variable $w_{f}$ to be mapped into a quantized domain using $(b+1)$ bits.
Let $\mathcal{Q}$ be a set of $(2k + 1)$ quantized values, where $k=2^b-1$.
Considering linear quantization, $\mathcal{Q}$ can be represented as $\left\{ -1, -\frac{k-1}{k}, . . ., -\frac{1}{k}, 0, \frac{1}{k}, . . ., \frac{k-1}{k}, 1 \right\}$, where $\frac{1}{k}$ is the size of the quantization bin.
Now, $w_{f}$ can be mapped to the $b$-bit quantization~\cite{Zhou2016DoReFaNetTL} space as follows:
\begin{equation}
w_{qo} = 2\times\quantize_b\left(\frac{\tanh(w_f)}{2\max(|\tanh(W_f)|)} + \frac{1}{2}\right) - 1
\label{eq:quant}
\end{equation}
where $\quantize_b(x) = \frac{1}{2^b-1}\round((2^b-1)x)$, $w_f$ is a scalar, $W_f$ is a vector, and $w_{qo}$ is a scalar in the range $[-1,1]$. 
Then, practically, a scaling factor $c$ is determined per layer to map the final quantized weight $w_{q}$ into the range $[-c,+c]$.
As such, $w_{q}$ takes the form $cw_{qo}$, where $c>0$, and $w_{qo} \in \mathcal{Q}$.

\niparagraph{Learning the sinusoidal period.} 
The parameter $\beta_i$ controls the period of the sinusoidal regularizer for layer $i$, thereby \textbf{$\beta_i$} is directly proportional to the actual quantization bitwidth \textbf{($b_i$)} of layer $i$ as follows: 
\begin{equation}
b_i = \ceil*{\beta_i } , \ \ \  and \ \ \ \alpha_i = b_i / \beta_i \label{eq:beta_alpha}
 \end{equation}
where \textbf{$\alpha_i \in \mathbb{R^+}$} is a scaling factor. 
Note that $b_i \in \mathbb{Z}$ is the only discrete parameter, while $\beta_i \in \mathbb{R^+}$ is a continuous real valued variable, and $\ceil*{.}$ is the ceiling operator.
%
%
While the first term in Equation~(\eqref{eq:sinreq}) is only responsible for promoting quantized weights, the second term enforces small bitwidths achieving a good accuracy-quantization trade-off.
The main insight here is that the sinusoidal period is a continuous valued parameter by definition.
As such, $\beta_i$ that defines the period serves as an ideal optimization objective and a proxy to minimize the actual quantization bitwidth $b_i$.
Therefore, \sinreq avoids the issues of gradient-based optimization for discrete valued parameters. 
Furthermore, the benefit of learning the sinusoidal period is two-fold.
First, it provides a smooth differentiable objective for finding minimal bitwidths.
Second, simultaneously learning the scaling factor ($\alpha_i$) associated with the found bitwidth.
%

%


%
\niparagraph{Putting it all together.}
Leveraging the sinusoidal properties, \sinreq learns the following two quantization parameters simultaneously: (i) a per-layer quantization bitwidth ($b_i$) along with (ii) a scaling factor ($\alpha_i$) through learning the period of the sinusoidal function.
Additionally, by exploiting the periodicity, differentiability, and the local convexity profile in sinusoidal functions \sinreq automatically propels network weights towards values that are inherently closer to quantization levels according to the jointly learned quantizer's parameters $b_i$, $\alpha_i$ defined in Equation~\eqref{eq:beta_alpha}.
%
%
These learned parameters can be mapped to the quantizer parameters explained for Equation~\eqref{eq:quant} in paragraph \textbf{Quantizer.}.
For $(b+1)$\footnote{the extra bit is the sign bit.} bits quantization, $k$ is set to $2^b-1$ and $c$ is set to $2^\alpha$.

\begin{figure}
  \centering 
  \includegraphics[width=0.45\textwidth]{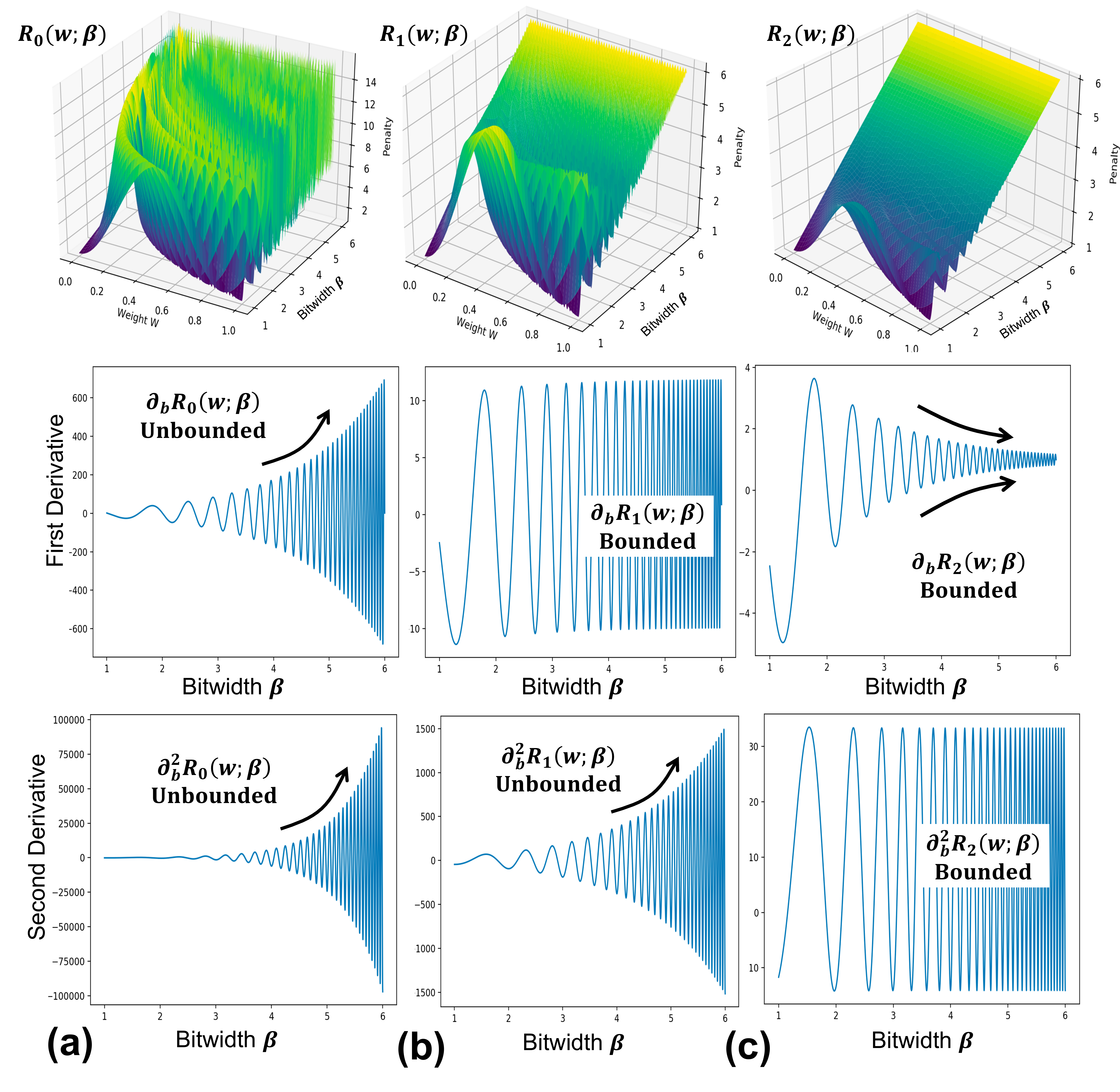}
    \vspace{-0.3cm}
  \caption{Visualization for three variants of the proposed regularization objective 
  		using different normalizations and their respective first and second derivatives with respect to $\beta$. (a) $R_0(w;\beta)$, (b) $R_1(w;\beta)$, and (c) $R_2(w;\beta)$.}
  \label{fig:deriv}
  \vspace{-0.5cm}
\end{figure}
\niparagraph{Bounding the gradients.} 
 The denominator in the first term of equation \eqref{eq:sinreq} is used to control the range of variation of the derivatives of the proposed regularization term with respect to $\beta$ and is chosen to limit vanishing and exploding gradients during training. To this end, we compared three variants of equation \eqref{eq:sinreq} with different normalization defined, for $k = 0$, $1$, and $2$, as:
\begin{gather}
 R_k(w;\beta) = \lambda_{w} \sum_{i}\sum_{j} \frac{\sin^{2}\left({\pi w_{ij}{(2^{\beta_i}-1)}}\right)}{2^{k \beta_i}}+ {\lambda_{\beta} \sum_{i} \beta_i}
\end{gather}   
Figure \ref{fig:deriv} (a), (b), (c) provide a visualization on how each of the proposed scaled variants impact the first and second derivatives.
For $R_0$ and $R_2$, there are regions of vanishing or exploding gradients. Only the regularization $R_1$ (the proposed one) is free of such issues. 
%
%

\niparagraph{Setting the regularization strengths.}
The convergence behavior depends on the setting of the regularization strengths $\lambda_{w}$ and $\lambda_{\beta}$.
Since our proposed objective seeks to learn multiple quantization parameterization in conjunction, we divide the learning process into three phases, as shown in Figure \ref{fig:function} (e).
In Phase (\ballnumber{1}), we primarily focus on optimizing for the original task loss $E_0$. 
Initially, the small $\lambda_{w}$ and $\lambda_{\beta}$ values allow the gradient descent to explore the optimization surface freely.
As the training process moves forward, we transition to phase (\ballnumber{2}) where the larger $\lambda_{w}$ and $\lambda_{\beta}$ gradually engage both the weights quantization regularization and the bitwidth regularization, respectively. 
Note that, for this to work, the strength of the weights quantization regularization $\lambda_{w}$ should be higher than the strength of the bitwidth regularization $\lambda_{\beta}$ such that a bitwidth per layer could be properly evaluated and eventually learned during this phase.
After the bitwidth regularizer converges to a bitwidth for each layer, we transition to phase (\ballnumber{3}), where we fix the learned bitwidths and gradually decay $\lambda_{\beta}$ while we keep $\lambda_{w}$ high.
In our experiments, we choose $\lambda_{w}$ and $\lambda_{\beta}$ such that the original loss and the penalty terms have approximately the same
magnitude.
%
It is worth noting that this way of progressivly setting the regularization strenghts across a multi-phase optimization resembles the settings of classical optimization algorithms, e.g. simulated annealing, where the temperature is progressively decreased from an initial positive value to zero or transitioning from exploration to exploitation.
The mathematical formula used to generate $\lambda_{w}$ and $\lambda_{\beta}$ profiles across iterations can be found in the appendix (Fig. 9).

\vspace{-0.4cm}
\section{Theoretical Analysis}
The results of this section are motivated as follow. 
Intuitively, we would like to show that the global minima of $E = E_0 + R$ are very close to the minima of $E_0$ that minimizes $R$. In other words, we expect to extract among the original solutions, the ones that are most prone to be quantized.
To establish such result, we will not consider the minima of $E = E_0 + R$, but the sequence $S_n$ of minima of $E_n = E_0 + \delta_n R$ defined for any sequence $\delta_n$ of real positive numbers. The next theorem shows that our intuition holds true, at least asymptotically with $n$ provided $\delta_n \to 0$.
\begin{theorem}\label{SetOfMinimaConverge}
Let $E_0,R:\mathbb{R}^n\rightarrow[0,\infty)$ be continuous and assume that the set $S_{E_0}$ of the global minima of $E_0$ is non-empty and compact. As $S_{E_0}$ is compact, we can also define $S_{E_0,R} \subseteq S_{E_0}$ as the set of
minima of $E_0$ which minimizes $R$. Let $\delta_n$ be a sequence of real positive numbers, define $E_n = E_0 + \delta_n R$ and the sequence $S_n = S_{E_n}$ of the global minima of $E_n$.
Then, the following holds true:
\begin{enumerate}
\item If $\delta_n\rightarrow 0$ and $S_n \rightarrow S_*$, then $S_*\subseteq S_{E_0,R}$,

\item If $\delta_n\rightarrow 0$ then there is a subsequence $\delta_{n_k}\rightarrow 0$ and a non-empty set $S_*\subseteq S_{E_0,R}$ so that $S_{n_k}\rightarrow S_*$,
\end{enumerate}
where the convergence of sets, denoted by $S_n \rightarrow S_*$, is defined as the
convergence to $0$ of their Haussdorff distance, i.e., $\displaystyle \lim_{n \to \infty} d_H(S_n, S_*) = 0$. 
\end{theorem}

\begin{proof}
For the first statement, assume that $S_{n}\rightarrow S_*$. We wish to show that $S_{*}\subseteq S_{E_0,R}$. Assume that $x_{n}$ is a sequence of global minima of $F+\delta_{n}G$ converging to $x_*$. It suffices to show that $x_*\in S_{E_0,R}$. First let us observe that $x_*\in S_{E_0}$. Indeed, let \[\lambda=\inf_{x\in \mathbb{R}^n} E_0(x)\] and assume that $x\in S_{E_0}$. Then, 
\[\lambda\leq E_0(x_{n})\leq (E_0+\delta_n R)(x_{n})\leq (E_0+\delta_n R)(x)=\underbrace{\lambda+\delta_n R(x)}_{\rightarrow \lambda}.\] Thus, since $E_0$ is continuous and $x_n\rightarrow x_*$ we have that $E_0(x_*)=\lambda$ which implies $x_*\in S_{E_0}$. Next, define \[\mu=\inf_{x\in S_{E_0}} R(x).\] Let $\hat x\in S_{E_0,R}$ so that $R(\hat x)=\mu$. Now observe that, by the minimality of $x_n$ we have that \[\lambda+\delta_n\mu=(E_0+\delta_n R)(\hat x)\geq (E_0+\delta_n R)(x_n)\geq \lambda+\delta_n R(x_n)\] Thus, 
$R(x_n)\leq \mu$ for all $n$. Since $R$ is continuous and $x_n\rightarrow x_*$ we have that $R(x_*)\leq \mu$ which implies that $R(x_*)=\mu$ since $x_*\in S_{E_0}$. Thus, $x_*\in S_{E_0,R}$.
The second statement follows from the standard theory of Hausdorff distance on compact metric spaces and the first statement. 
\end{proof}
\vspace{-0.3cm}
Theorem \ref{SetOfMinimaConverge} implies that by decreasing the strength of $R$, one recovers the subset of the original solutions that achives the smallest quantization loss. In practice, we are not interested in global minima, and we should not decrease much the strength of $R$. In our context, Theorem \ref{SetOfMinimaConverge} should then be understood as a proof of concept on why the proposed approach leads the expected result. Experiments carried out in the next section will support this claim.
For the interested reader, we provide a more detailed version of the above analysis in the Appendix B.
\begin{table*}[!ht]
	\centering
	\caption{Comparison with state-of-the-art quantization methods on ImageNet. 
			The `` W/A ''  values are the bitwidths of weights/activations.}
	\includegraphics[width=0.8\linewidth]{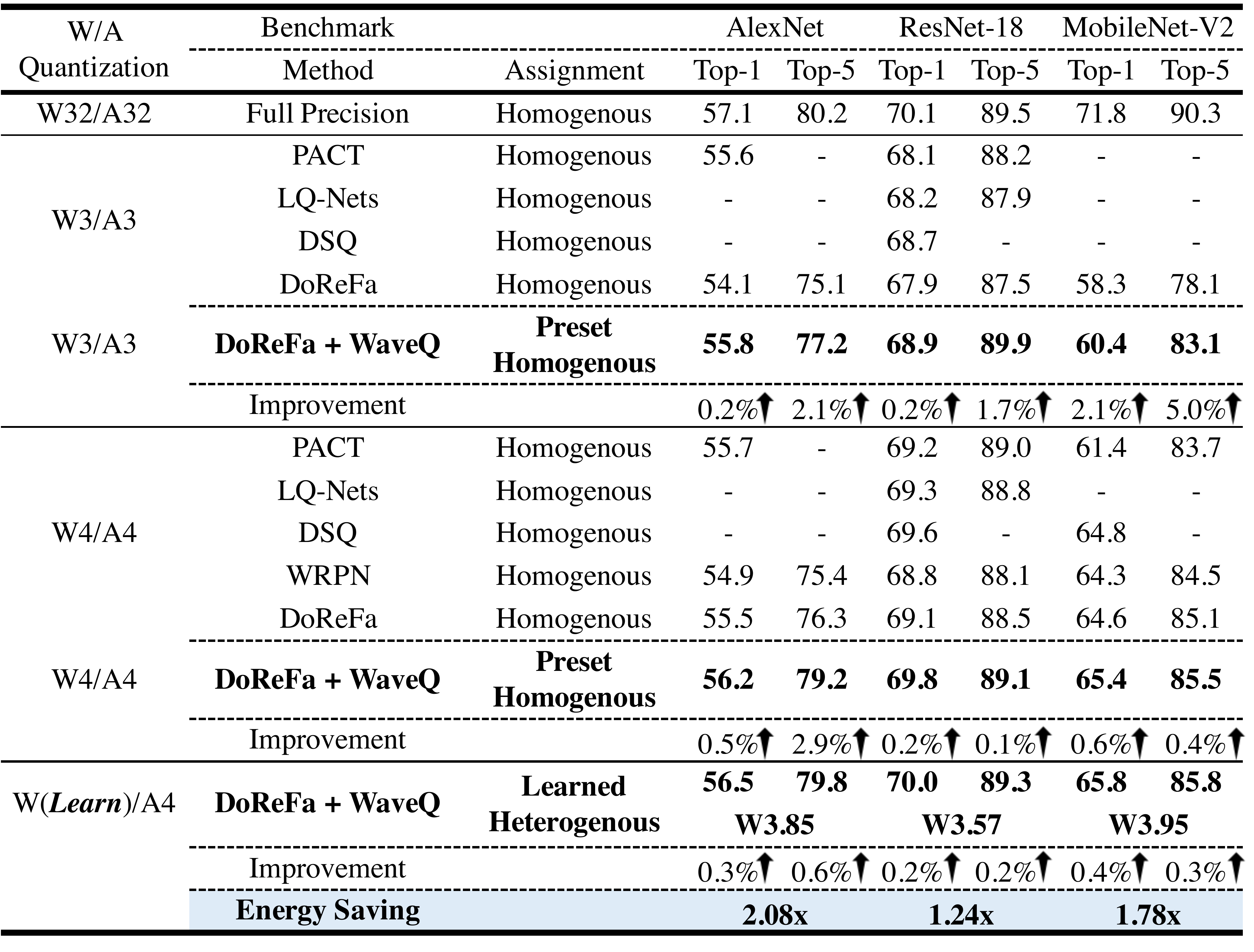}
	\label{table:SOTA}
\end{table*}
%
\vspace{-0.2cm}
\section{Experimental Results} 
\begin{figure*}[t]
  \centering
  \includegraphics[width=0.9\textwidth]{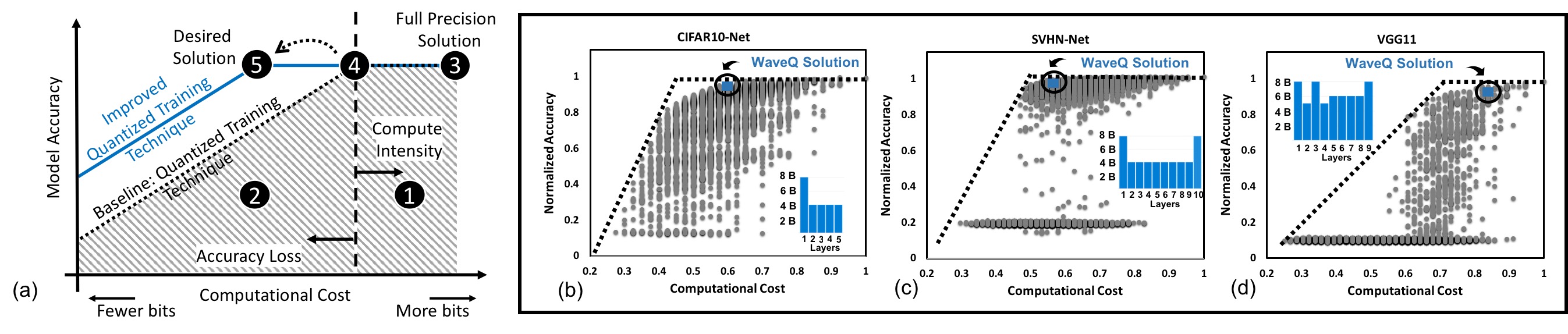}
  \caption{Quantization space in terms of computation and accuracy for (a) CIFAR-10, (b) SVHN, and (c) VGG-11}
  \label{fig:enum}
\end{figure*}
To demonstrate the effectiveness of our proposed \sinreq, we evaluated it on several deep neural networks with different image classification datasets (CIFAR10, SVHN, and ImageNet).
%
%
We provide results for two different types of quantization. 
First, we show quantization results for \emph{learned} heterogenous bitwidths using \sinreq and we provide different arguments to asses the quality of these learned bitwidth assignments.
Second, we further provide results assuming a \emph{preset} homogenous bitwidth assignment as a special setting of \sinreq, 
which in some cases is a practical assumption that might stem from particular hardware requirements or constraints.
Table~\ref{table:SOTA} provides a summary of the evaluated networks and datasets for both learned heterogenous bitwidths, and the special case of training preset homogenous bitwidth assignments.
We compare our proposed \sinreq method with PACT~\cite{Choi2018PACTPC}, LQ-Nets~\cite{DBLP:conf/eccv/ZhangYYH18}, DSQ~\cite{DBLP:journals/corr/abs-1908-05033},  and DoReFa, which are current state-of-the-art (SOTA) methods that show results with 3-, and 4-bit weight/activation quantization for various networks architectures (AlexNet, ResNet-18, and MobileNet).

%
%
%
%

%
\subsection{Experimental Setup}
%
We implemented our technique inside Distiller~\cite{neta_zmora_2018_1297430}, an open source framework for compression by Intel Nervana.
The reported accuracies for DoReFa and WRPN are with the built-in implementations in Distiller, which may not exactly match the accuracies reported in their respective papers.
However, an independent implementation from a major company provides an unbiased foundation for the comparisons.
We consider quantizing all convolution and fully connected layers, except for the first and last layers which may use higher precision. 

\subsection{Learned Heterogenous Bitwidth Quantization}
\niparagraph{Quantization levels with \sinreq.}
As for quantizing both weights and activations, Table~\ref{table:SOTA} shows that incorporating \sinreq into the quantized training process yields best accuracy results outperforming PACT, LQ-Net, DSQ, and DoReFa with significant margins.
Furthermore, it can be seen that the learned heterogenous btiwidths yield better accuracy as compared to the preset 4-bit homogenous assignments, with lower, on average, bitwidh (3.85-, 3.57-, and 3.95- bits for AlexNet, ResNet-18, and MobileNet, respectively).

Figure \ref{fig:bars} (a), (b) (bottom bar graphs) show the learned heterogenous weight bitwidths over layers for AlexNet and ResNet-18, respectively.
As can be seen, \sinreq objective learning shows a spectrum of varying bitwidth assignments to the layers which vary from 2 bits to 8 bits with an irregular pattern.
These results demonstrate that the proposed regularization objective, \sinreq, automatically distinguishes different layers and their varying importance with respect to accuracy while learning their respective bitwidths.

Although, we can observe slight correlation of learning small bitwidths for layers with many parameters, e.g., fully connected layers, due to the strong inter-dependence between layers of neural networks, the resulting bitwidth assignments are generally complex, thereby there is no simple heuristic that can be deduced. 
As such, it is important to develop techniques to automatically learn a near-optimal bitwidth assignment for a given deep neural network.
To assess the quality of these bitwidths assignments, we conduct a sensitivity analysis to the relatively big networks, and 
a Pareto analysis on the DNNs for which we could populate the search space as shown below.

\niparagraph{Superiority of heterogenous quantization.}
Figure \ref{fig:bars} (a), (b) (top graphs) show various comparisons and sensitivity results for learned heterogenous bitwidth assignments for bigger networks (AlexNet and ResNet-18) that are infeasible to enumerate their respective quantization spaces.
Compared to 4-bit homogenous quantization, it can be seen that learned heterogenous assignments achieve better accuracy  with lower, on average, bitwidth $~3.85$ bits for AlexNet and $~3.57$ bits for ResNet-18.
This demonstrates that a homogenous (uniform) assignment of the bits is not always the desired choice to preserve accuracy.
Furthermore, Figure \ref{fig:bars} also shows that decrementing the learned bitwidth for any single layer at a time results in $0.44\%$ and $0.24\%$ reduction in accuracy on average (across all layers of the network) for AlexNet and ResNet-18, respectively, which further demonstrates the learning quality of \sinreq.

\niparagraph{Validation: Pareto analysis.}
Figure \ref{fig:enum} (a) shows a sketch of the multi-objective optimization problem of layer-wise quantization of a neural network showing the underlying design space and the different design components. 
Given a particular architecture and a training technique, different combinations of layer-wise quantization bitwidths form a network specific design space (possible solutions). 
The design space can be divided into two regions. 
Region \ballnumber{1} represents the set of combinations of layer-wise quantization bitwidths that preserves the accuracy.
On the other side, region \ballnumber{2} represents the set of all the remaining combinations of layer-wise quantization bitwidths that are associated with some sort of accuracy loss.
As number of bits (on average across layers) increases, the amount of compute increases (considering full precision solution (\ballnumber{3}) corresponds to the max amount of compute).
The objective is to find the least (on average) combination of bitwidths that still preserves the accuracy (i.e., solution \ballnumber{4}; the solution at the interface of the two regions).
As mentioned, that is for a particular training technique. Utilizing an improved quantized training technique modifies the envelop of the design space by expanding region \ballnumber{1}.
In other words, it pushes the desired solution (initially \ballnumber{4}) to a lower combination of bitwidths \ballnumber{5} (i.e., preserves the accuracy with lower bitwidths on average).
Note that this is just a sketch for illustration purposes and does not imply that the actual envelope of region \ballnumber{2} is linear.
%
Figure \ref{fig:enum} (b)-(d) depicts actual solutions spaces for three benchmarks (CIFAR10, SVHN, and VGG11) in terms of computation versus accuracy.
Each point on these charts is a unique combination of bitwidths that are assigned to the layers of the network.
The boundary of the solutions denotes the Pareto frontier and is highlighted by a dashed line.
The solution found by \sinreq is marked out using an arrow and lays on the desired section of the Pareto frontier where the compute intensity is minimized while keeping the accuracy intact, which demonstrates the quality of the obtained solutions.
%
%
It is worth noting that as a result of the moderate size of the three networks presented in this subsection, it was possible to enumerate the design space, obtain Pareto frontier and assess \sinreq quantization policy for each of the three networks. 
However, it is infeasible to do so for state-of-the-art deep networks (e.g., AlexNet and ResNet) which further stresses the importance of automation and efficacy of \sinreq.
\begin{figure*}[t]
  \centering 
  \includegraphics[width=0.85\textwidth]{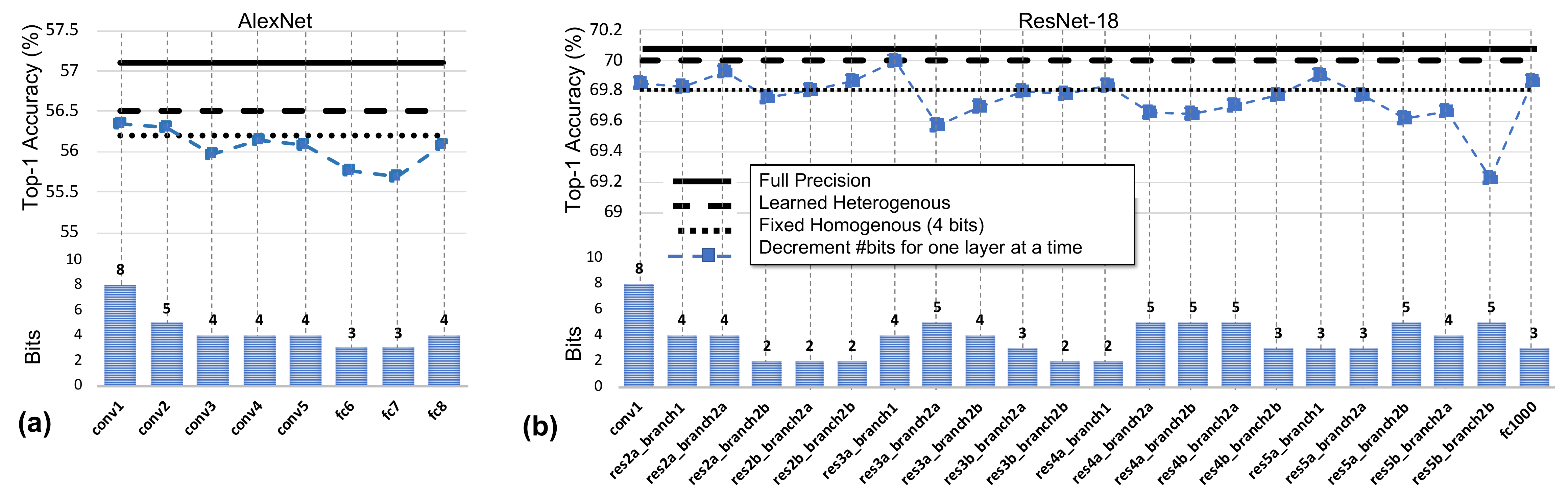}
  \caption{Quantization bitwidth assignments across layers. It can be seen that learned heterogenous assignments achieve 
  		(with lower on average bitwidth) better accuracy as compared to fixed homogenous assignments.
		(a) AlexNet (average bitwidth = 3.85 bits). (b) ResNet-18 (average bitwidth = 3.57 bits)}
  \label{fig:bars}
\end{figure*}
\begin{figure}
  \centering
  \includegraphics[width=0.45\textwidth]{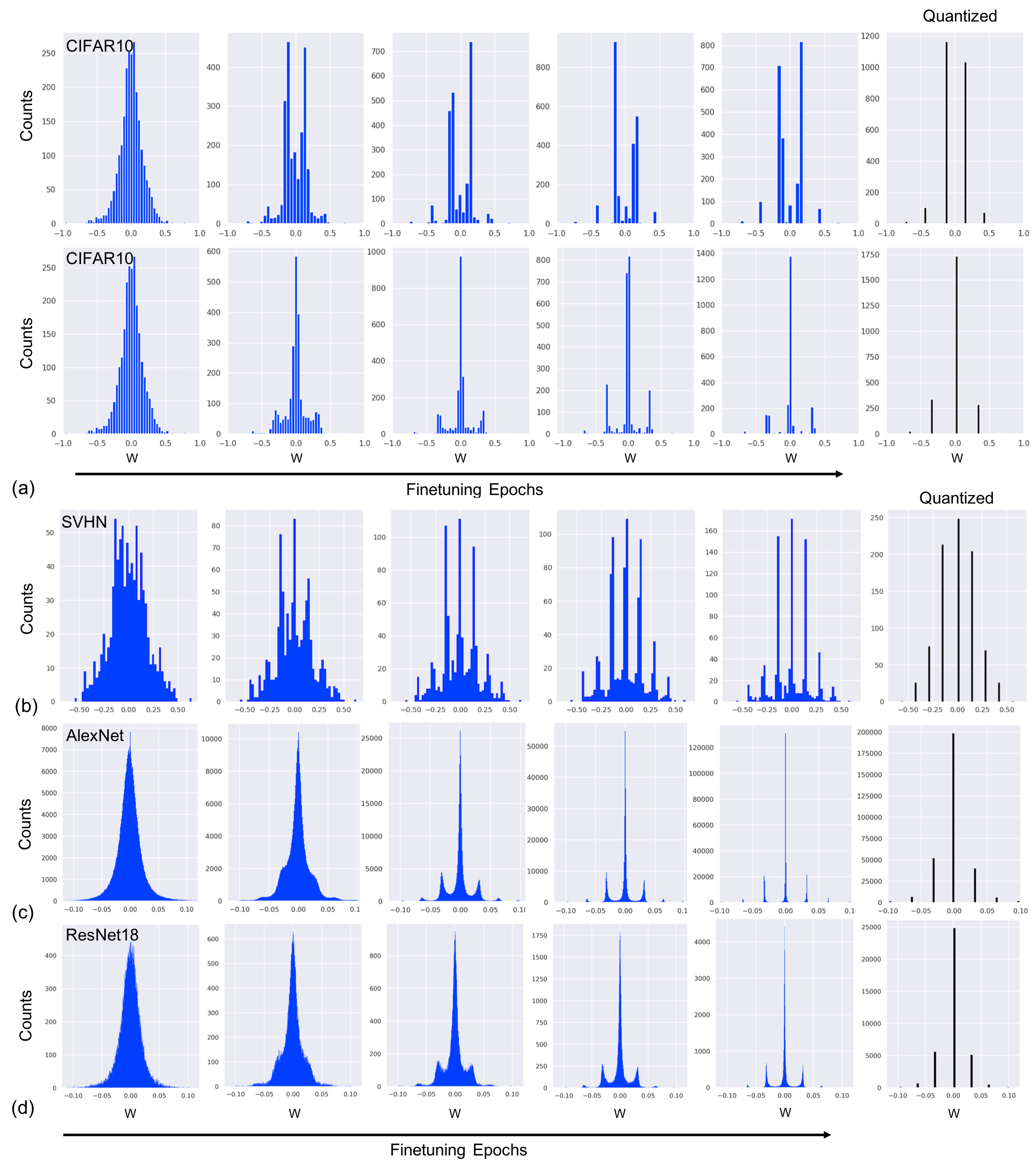}
  \caption{{Evolution of weight distributions over training epochs (with the proposed regularization) at different layers and bitwidths for different networks}. (a) CIFAR10, second convolution layer with 3 bits, top row: mid-rise type of quantization (shifting by half a step to exclude zero as a quantization level); bottom row: mid-tread type of quantization (zero is included as a quantization level). (b) SVHN, top row: first convolution layer with 4 bits quantization. (c) AlexNet, second convolution layer with 4 bits quantization, and (d) ResNet-18, , second convolution layer with 4 bits quantization.}
\vspace{-0.3cm}
  \label{fig:q_w_dist}
\end{figure}
\begin{table}
	\centering
	\caption{Comparing accuracies of different networks using plain WRPN, plain DoReFa and DoReFa + \sinreq on fixed homogenous weight quantization.}
	\includegraphics[width=1.0\linewidth]{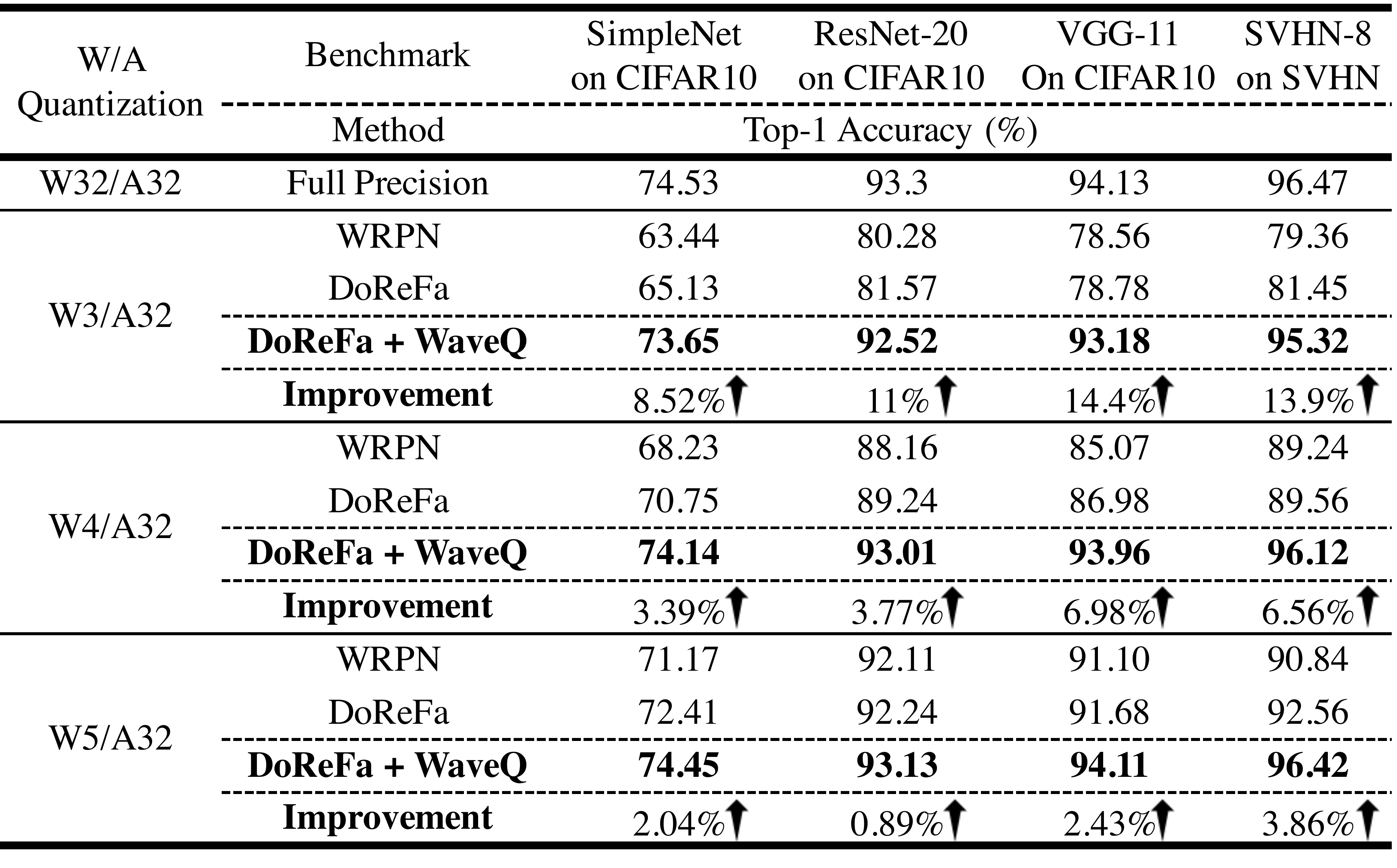}
	\label{table:dorefa_wrpn}
	\vspace{-0.6cm}
\end{table}
%
%

\vspace{-0.1cm}
\niparagraph{Energy savings.}
To further demonstrate the energy savings of the solutions found by \sinreq, we evaluate it on Stripes~\cite{DBLP:conf/micro/JuddAHAM16}, 
a custom accelerator designed for DNNs, which exploits bit-serial computation to support flexible bitwidths for DNN operations.
As shown in Table~\ref{table:SOTA}, the reduction in the bitwidth, on average, leads to 77.5\% reduction in the energy consumed during the execution of these networks.

\subsection{Preset Homogenous Bitwidth Quantization}
Now, we consider a preset homogenous bitwidth quantization which can also be supported by the proposed \sinreq under special settings where we fix $\beta$ (to a preset bitwidth), thus only the first regularization term is engaged for weight quantization.
Table~\ref{table:dorefa_wrpn} shows results comparison of different networks (SimpleNet-5, ResNet-20, VGG-11, and SVHN-8) using plain WRPN, plain DoReFa and DoReFa + \sinreq considering preset 3-, 4-, and 5- bitwdith assignments.
As can be seen, 
These results concretely show the impact of incorporating \sinreq into existing quantized training techniques and how it outperforms previously reported accuracies of several SOTA methods.

\niparagraph{Semi-quantized weight distributions.}
Figure \ref{fig:q_w_dist} shows the evolution of weights distributions over fine-tuning epochs for different layers of CIFAR10, SVHN, AlexNet, and ResNet-18 networks. The high-precision weights form clusters and gradually converge around the quantization centroids as regularization loss is minimized along with the main accuracy loss.

\vspace{-0.3cm}
\section{Discussion}
\begin{figure}
  \centering
	\includegraphics[width=1\linewidth]{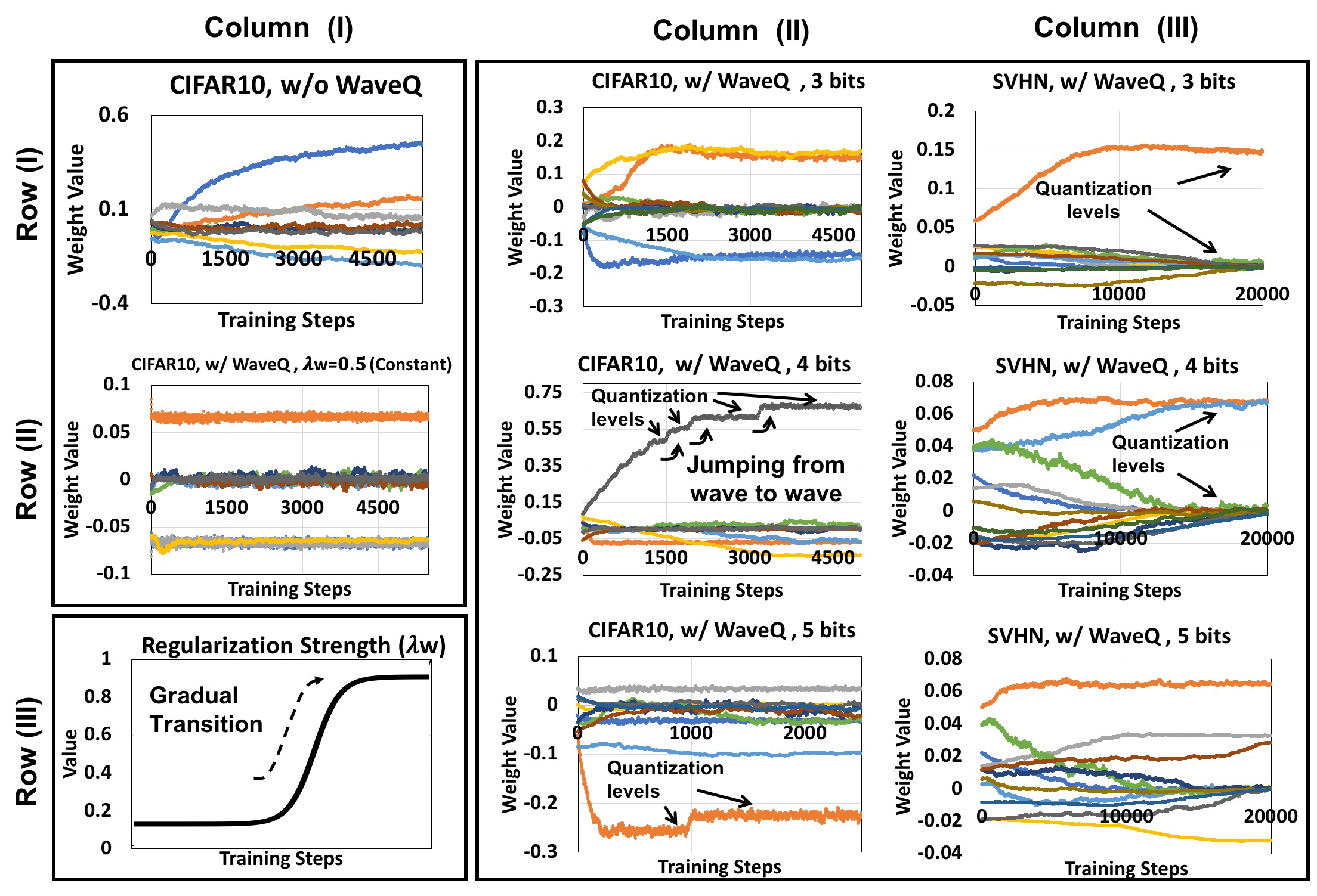}	
\vspace{-0.65cm}
	\caption{Weight trajectories. The 10 colored lines in each plot denote the trajectory of 10 different weights.}
\vspace{-0.65cm}
  \label{fig:w_traj}
\end{figure}
%

%
%

%
We conduct an experiment that uses \sinreq for training from scratch.
For the sake of clarity, we are considering in this experiment the case of preset bitwidth assignments (i.e., $\lambda_\beta = 0$).
Figure~\ref{fig:w_traj}-Row(I)-Column(I) shows weight trajectories without \sinreq as a point of reference.
Row(II)-Column(I) shows the weight trajectories when \sinreq is used with a constant $\lambda_w$.
As Figure~\ref{fig:w_traj}-Row(II)-Column(I) illustrates, using a constant $\lambda_w$ results in the weights being stuck in a region close to their initialization, (i.e., quantization objective dominates the accuracy objective). 
However, if we dynamically change the $\lambda_w$ following the exponential curve in Figure~\ref{fig:w_traj}-Row(III)-Column(I)) during the from-scratch training, the weights no longer get stuck.
Instead, the weights traverse the space (\emph{i.e., jump from wave to wave}) as illustrated in Figure~\ref{fig:w_traj}-Columns(II) and (III) for CIFAR and SVHN, respectively.
In these two columns, Rows (I), (II), (III), correspond to quantization with 3, 4, 5 bits, respectively.
Initially, the smaller $\lambda_w$ values allow the gradient descent to explore the optimization surface freely, as the training process moves forward, the larger $\lambda_w$ gradually engages the sinusoidal regularizer, and eventually pushes the weights close to the quantization levels.
Further convergence analysis is provided in the Appendix A.
%
%

\vspace{-0.1cm}
\section{Related Work}

This research lies at the intersection of (1) quantized training algorithms and (2) techniques that discover bitwidth for quantization.
The following diuscusses the most related works in both directions.
In contrast, \sinreq modifies the loss function of the training to simultaneously \emph{learn} the period of an adaptive sinusoidal regularizer through the same stochastic gradient descent that trains the network.
The diffrentionablity of the adaptive sinusoidal regaulrizer enables simultaneously learning both the bitwidthes and pushing the wight values to the quantization levels.
As such, \sinreq can be used as a complementary method to some of these efforts, which is demonstrated by experiments with both DoReFa-Net~\cite{Zhou2016DoReFaNetTL} and WRPN~\cite{Mishra2017WRPNWR}.

Our preliminary efforts~\cite{sinreq_icml19_workshop} and another work concurrent to it~\cite{DBLP:journals/corr/abs-1811-09862} use a sinusoidal regularization to push the weights closer to the quantization levels.
However, neither of these two works make the period a differentiable parameter nor find bitwidths during training.

\niparagraph{Quantized training algorithms}
There have been several techniques \cite{Zhou2016DoReFaNetTL,Zhu2016TrainedTQ,Mishra2017WRPNWR} that train a neural network in a quantized domain after the bitwidth of the layers is determined manually.
%
DoReFa-Net~\cite{Zhou2016DoReFaNetTL} uses straight through estimator~\cite{DBLP:journals/corr/BengioLC13} for quantization and extends it for any arbitrary $k$ bit quantization.
DoReFa-Net generalizes the method of binarized neural networks to allow creating a
CNN that has arbitrary bitwidth below 8 bits in weights, activations, and gradients. 
WRPN~\cite{Mishra2017WRPNWR} is  training algorithm that compensates for the reduced precision by increasing the number of filter maps in a layer (doubling or tripling). 
%
TTQ~\cite{Zhu2016TrainedTQ} quantizes the weights to ternary values by using per layer scaling coefficients that are learnt during training. These scaling coefficients are used to scale the weights during inference.
%
PACT~\cite{Choi2018PACTPC} proposes a technique for quantizing activations by introducing an activation clipping parameter $\alpha$. This parameter ($\alpha$) is used to represent the clipping level in the activation function and is learned via back-propagation during training.
More recently, VNQ~\cite{DBLP:conf/iclr/AchterholdKSG18} uses a variational Bayesian approach for quantizing neural network weights during training. 
%
DCQ~\cite{dcq_icml19_workshop} employs sectional multi-backpropagation algorithm that leverages multiple instances of knowledge distillation and
intermediate feature representations to teach a quantized student through divide and conquer.

\niparagraph{Loss-aware weight quantization.}
Recent works pursued loss-aware minimization approaches for quantization. 
\cite{DBLP:conf/iclr/HouYK17} and \cite{DBLP:conf/iclr/HouK18} developed approximate solutions using proximal Newton algorithm to minimize the loss function directly under the constraints of low bitwidth weights.
One effort~\cite{DBLP:journals/corr/abs-1809-00095} proposed to learn the quantization of DNNs through a regularization term of the mean-squared-quantization
error. 
LQ-Net~\cite{DBLP:conf/eccv/ZhangYYH18} proposes to jointly train the network and its quantizers.
The quantizer is a inner product between a basis vector and the binary coding vector.
%
%
%
DSQ~\cite{DBLP:journals/corr/abs-1908-05033} employs a series of tanh functions to gradually approximate the staircase
function for low-bit quantization (e.g., sign for 1-bit case),
and meanwhile keeps the smoothness for easy gradient calculation.
Although some of these techniques use regularization to guide the process of quantized training, none explores the use of adaptive sinusoidal regularizers for quantization.
Moreover, unlike \sinreq, these techniques do not find the bitwidth for quantizing the layers.

%
%
%
%

\niparagraph{Techniques for discovering quantization bitwidths.}
A recent line of research focused on methods which can also find the optimal quantization parameters, e.g., the bitwidth, the stepsize, in parallel to the network weights.
Recent work~\cite{Ye2018AUF} based on ADMM runs a binary search to minimize the total square quantization error in order to decide the quantization levels for the layers. 
They use a heuristic-based iterative optimization technique for fine-tuning.
%
%
Most recently, ~\cite{DBLP:journals/corr/abs-1905-11452} proposed to indirectly learn quantizer's parameters via Straight Through Estimator (STE)~\cite{DBLP:journals/corr/BengioLC13} based approach.
In a similar vein, \cite{DBLP:journals/corr/abs-1902-08153} has proposed to learn the quantization mapping for each layer in a deep network by approximating the gradient to the quantizer step size that is sensitive to quantized state transitions.
On another side, recent works~\cite{releq_nips19_workshop} proposed a reinforcement learning based approach to find an optimal bitwidth assignment policy.


\vspace{-0.3cm}
\section{Conclusion}
This paper provided a new approach in using sinusoidal regularizations to cast the two problems of finding bitwidth levels for layers and quantizing the weights as a gradient-based optimization through sinusoidal regularization.
While this technique consistently improves the accuracy, \sinreq does not require changes to the base training algorithm or the neural network topology.

%
%

\bibliography{reference/paper}

\begin{thebibliography}{27}
\providecommand{\natexlab}[1]{#1}
\providecommand{\url}[1]{\texttt{#1}}
\expandafter\ifx\csname urlstyle\endcsname\relax
  \providecommand{\doi}[1]{doi: #1}\else
  \providecommand{\doi}{doi: \begingroup \urlstyle{rm}\Url}\fi

\bibitem[Achterhold et~al.(2018)Achterhold, K{\"{o}}hler, Schmeink, and
  Genewein]{DBLP:conf/iclr/AchterholdKSG18}
Achterhold, J., K{\"{o}}hler, J.~M., Schmeink, A., and Genewein, T.
\newblock Variational network quantization.
\newblock In \emph{6th ICLR}, 2018.

\bibitem[Bengio et~al.(2013)Bengio, L{\'{e}}onard, and
  Courville]{DBLP:journals/corr/BengioLC13}
Bengio, Y., L{\'{e}}onard, N., and Courville, A.~C.
\newblock Estimating or propagating gradients through stochastic neurons for
  conditional computation.
\newblock \emph{CoRR}, abs/1308.3432, 2013.
\newblock URL \url{http://arxiv.org/abs/1308.3432}.

\bibitem[Choi et~al.(2018{\natexlab{a}})Choi, Wang, Venkataramani, Chuang,
  Srinivasan, and Gopalakrishnan]{Choi2018PACTPC}
Choi, J., Wang, Z., Venkataramani, S., Chuang, P. I.-J., Srinivasan, V., and
  Gopalakrishnan, K.
\newblock Pact: Parameterized clipping activation for quantized neural
  networks.
\newblock \emph{CoRR}, abs/1805.06085, 2018{\natexlab{a}}.

\bibitem[Choi et~al.(2018{\natexlab{b}})Choi, El{-}Khamy, and
  Lee]{DBLP:journals/corr/abs-1809-00095}
Choi, Y., El{-}Khamy, M., and Lee, J.
\newblock Learning low precision deep neural networks through regularization.
\newblock \emph{CoRR}, abs/1809.00095, 2018{\natexlab{b}}.

\bibitem[Choromanska et~al.(2015)Choromanska, Henaff, Mathieu, Arous, and
  LeCun]{DBLP:journals/corr/ChoromanskaHMAL14}
Choromanska, A., Henaff, M., Mathieu, M., Arous, G.~B., and LeCun, Y.
\newblock {The Loss Surfaces of Multilayer Networks}.
\newblock In \emph{Artificial Intelligence and Statistics}, 2015.

\bibitem[Elthakeb et~al.(2018)Elthakeb, Pilligundla, Mireshghallah,
  Yazdanbakhsh, and Esmaeilzadeh]{releq_nips19_workshop}
Elthakeb, A.~T., Pilligundla, P., Mireshghallah, F., Yazdanbakhsh, A., and
  Esmaeilzadeh, H.
\newblock Re{L}e{Q}: {A} reinforcement learning approach for deep quantization
  of neural networks.
\newblock \emph{Advances in Neural Information Processing Systems (NeurIPS)
  Workshop on Machine Learning for Systems}, 2018.
\newblock URL \url{https://arxiv.org/abs/1811.01704}.

\bibitem[Elthakeb et~al.(2019{\natexlab{a}})Elthakeb, Pilligundla, and
  Esmaeilzadeh]{dcq_icml19_workshop}
Elthakeb, A.~T., Pilligundla, P., and Esmaeilzadeh, H.
\newblock Divide and {C}onquer: Leveraging intermediate feature representations
  for quantized training of neural networks.
\newblock \emph{International Conference on Machine Learning (ICML) Workshop on
  Understanding and Improving Generalization in Deep Learning},
  2019{\natexlab{a}}.
\newblock URL \url{https://arxiv.org/abs/1906.06033}.

\bibitem[Elthakeb et~al.(2019{\natexlab{b}})Elthakeb, Pilligundla, and
  Esmaeilzadeh]{sinreq_icml19_workshop}
Elthakeb, A.~T., Pilligundla, P., and Esmaeilzadeh, H.
\newblock Sin{R}e{Q}: Generalized sinusoidal regularization for low-bitwidth
  deep quantized training.
\newblock \emph{International Conference on Machine Learning (ICML) Workshop on
  Understanding and Improving Generalization in Deep Learning},
  2019{\natexlab{b}}.
\newblock URL \url{https://arxiv.org/abs/1905.01416}.

\bibitem[Esser et~al.(2019)Esser, McKinstry, Bablani, Appuswamy, and
  Modha]{DBLP:journals/corr/abs-1902-08153}
Esser, S.~K., McKinstry, J.~L., Bablani, D., Appuswamy, R., and Modha, D.~S.
\newblock Learned step size quantization.
\newblock \emph{8th ICLR, 2020}, abs/1902.08153, 2019.
\newblock URL \url{http://arxiv.org/abs/1902.08153}.

\bibitem[Gong et~al.(2019)Gong, Liu, Jiang, Li, Hu, Lin, Yu, and
  Yan]{DBLP:journals/corr/abs-1908-05033}
Gong, R., Liu, X., Jiang, S., Li, T., Hu, P., Lin, J., Yu, F., and Yan, J.
\newblock Differentiable soft quantization: Bridging full-precision and low-bit
  neural networks.
\newblock \emph{CoRR}, abs/1908.05033, 2019.
\newblock URL \url{http://arxiv.org/abs/1908.05033}.

\bibitem[Hou \& Kwok(2018)Hou and Kwok]{DBLP:conf/iclr/HouK18}
Hou, L. and Kwok, J.~T.
\newblock Loss-aware weight quantization of deep networks.
\newblock In \emph{6th ICLR}, 2018.

\bibitem[Hou et~al.(2017)Hou, Yao, and Kwok]{DBLP:conf/iclr/HouYK17}
Hou, L., Yao, Q., and Kwok, J.~T.
\newblock Loss-aware binarization of deep networks.
\newblock In \emph{5th ICLR}, 2017.

\bibitem[Hubara et~al.(2017)Hubara, Courbariaux, Soudry, El-Yaniv, and
  Bengio]{Hubara2017QNN}
Hubara, I., Courbariaux, M., Soudry, D., El-Yaniv, R., and Bengio, Y.
\newblock {Quantized Neural Networks: Training Neural Networks with Low
  Precision Weights and Activations}.
\newblock \emph{J. Mach. Learn. Res.}, 2017.

\bibitem[Judd et~al.(2016{\natexlab{a}})Judd, Albericio, Hetherington, Aamodt,
  and Moshovos]{DBLP:conf/micro/JuddAHAM16}
Judd, P., Albericio, J., Hetherington, T.~H., Aamodt, T.~M., and Moshovos, A.
\newblock Stripes: Bit-serial deep neural network computing.
\newblock In \emph{49th Annual {IEEE/ACM} International Symposium on
  Microarchitecture, {MICRO} 2016, Taipei, Taiwan, October 15-19, 2016}, pp.\
  19:1--19:12. {IEEE} Computer Society, 2016{\natexlab{a}}.
\newblock \doi{10.1109/MICRO.2016.7783722}.
\newblock URL \url{https://doi.org/10.1109/MICRO.2016.7783722}.

\bibitem[Judd et~al.(2016{\natexlab{b}})Judd, Albericio, Hetherington, Aamodt,
  and Moshovos]{Judd2016StripesBD}
Judd, P., Albericio, J., Hetherington, T.~H., Aamodt, T.~M., and Moshovos, A.
\newblock Stripes: Bit-serial deep neural network computing.
\newblock \emph{49th MICRO}, pp.\  1--12, 2016{\natexlab{b}}.

\bibitem[Li et~al.(2018)Li, Xu, Taylor, Studer, and
  Goldstein]{DBLP:journals/corr/abs-1712-09913}
Li, H., Xu, Z., Taylor, G., Studer, C., and Goldstein, T.
\newblock {Visualizing the Loss Landscape of Neural Nets}.
\newblock In \emph{NIPS}, 2018.

\bibitem[Mishra et~al.(2018)Mishra, Nurvitadhi, Cook, and
  Marr]{Mishra2017WRPNWR}
Mishra, A.~K., Nurvitadhi, E., Cook, J.~J., and Marr, D.
\newblock {WRPN: Wide Reduced-Precision Networks}.
\newblock In \emph{ICLR}, 2018.

\bibitem[Naumov et~al.(2018)Naumov, Diril, Park, Ray, Jablonski, and
  Tulloch]{DBLP:journals/corr/abs-1811-09862}
Naumov, M., Diril, U., Park, J., Ray, B., Jablonski, J., and Tulloch, A.
\newblock On periodic functions as regularizers for quantization of neural
  networks.
\newblock \emph{CoRR}, abs/1811.09862, 2018.
\newblock URL \url{http://arxiv.org/abs/1811.09862}.

\bibitem[Rumelhart et~al.(1986)Rumelhart, Hinton, and
  Williams]{rumelhart:errorpropnonote}
Rumelhart, D.~E., Hinton, G.~E., and Williams, R.~J.
\newblock Learning internal representations by error propagation.
\newblock In \emph{Parallel Distributed Processing: Explorations in the
  Microstructure of Cognition, {V}olume 1: {F}oundations}, pp.\  318--362. MIT
  Press, Cambridge, MA, 1986.

\bibitem[Sharma et~al.(2018)Sharma, Park, Suda, Lai, Chau, Chandra, and
  Esmaeilzadeh]{bitfusion}
Sharma, H., Park, J., Suda, N., Lai, L., Chau, B., Chandra, V., and
  Esmaeilzadeh, H.
\newblock Bit fusion: Bit-level dynamically composable architecture for
  accelerating deep neural network.
\newblock \emph{ISCA}, pp.\  764--775, 2018.

\bibitem[Uhlich et~al.(2019)Uhlich, Mauch, Yoshiyama, Cardinaux, Garc{\'{\i}}a,
  Tiedemann, Kemp, and Nakamura]{DBLP:journals/corr/abs-1905-11452}
Uhlich, S., Mauch, L., Yoshiyama, K., Cardinaux, F., Garc{\'{\i}}a, J.~A.,
  Tiedemann, S., Kemp, T., and Nakamura, A.
\newblock Mixed precision dnns: All you need is a good parametrization.
\newblock abs/1905.11452, 2019.
\newblock URL \url{http://arxiv.org/abs/1905.11452}.

\bibitem[Wang et~al.(2019)Wang, Liu, Lin, Lin, and
  Han]{DBLP:conf/cvpr/WangLLLH19}
Wang, K., Liu, Z., Lin, Y., Lin, J., and Han, S.
\newblock {HAQ:} hardware-aware automated quantization with mixed precision.
\newblock In \emph{{IEEE} Conference on Computer Vision and Pattern
  Recognition, {CVPR} 2019, Long Beach, CA, USA, June 16-20, 2019}, pp.\
  8612--8620. Computer Vision Foundation / {IEEE}, 2019.
\newblock \doi{10.1109/CVPR.2019.00881}.
\newblock URL
  \url{http://openaccess.thecvf.com/content\_CVPR\_2019/html/Wang\_HAQ\_Hardware-Aware\_Automated\_Quantization\_With\_Mixed\_Precision\_CVPR\_2019\_paper.html}.

\bibitem[Ye et~al.(2018)Ye, Zhang, Zhang, Li, Xie, Liang, Liu, Lin, and
  Wang]{Ye2018AUF}
Ye, S., Zhang, T., Zhang, K., Li, J., Xie, J., Liang, Y., Liu, S., Lin, X., and
  Wang, Y.
\newblock A unified framework of dnn weight pruning and weight
  clustering/quantization using admm.
\newblock \emph{CoRR}, abs/1811.01907, 2018.

\bibitem[Zhang et~al.(2018)Zhang, Yang, Ye, and Hua]{DBLP:conf/eccv/ZhangYYH18}
Zhang, D., Yang, J., Ye, D., and Hua, G.
\newblock Lq-nets: Learned quantization for highly accurate and compact deep
  neural networks.
\newblock In \emph{ECCV}, pp.\  373--390, 2018.
\newblock \doi{10.1007/978-3-030-01237-3\_23}.

\bibitem[Zhou et~al.(2016)Zhou, Ni, Zhou, Wen, Wu, and
  Zou]{Zhou2016DoReFaNetTL}
Zhou, S., Ni, Z., Zhou, X., Wen, H., Wu, Y., and Zou, Y.
\newblock {DoReFa-Net: Training Low Bitwidth Convolutional Neural Networks with
  Low Bitwidth Gradients}.
\newblock \emph{CoRR}, 2016.

\bibitem[Zhu et~al.(2017)Zhu, Han, Mao, and Dally]{Zhu2016TrainedTQ}
Zhu, C., Han, S., Mao, H., and Dally, W.~J.
\newblock {Trained Ternary Quantization}.
\newblock In \emph{ICLR}, 2017.

\bibitem[Zmora et~al.(2018)Zmora, Jacob, and Novik]{neta_zmora_2018_1297430}
Zmora, N., Jacob, G., and Novik, G.
\newblock Neural network distiller, June 2018.

\end{thebibliography}
\bibliographystyle{formattings/icml2019}

\clearpage
\appendix
\section{Convergence analysis}
Figure \ref{fig:conv} (a), (b) show the convergence behavior of \sinreq by visualizing both accuracy and regularization loss over finetuning epochs for two networks: CIFAR10 and SVHN.
As can be seen, the regularization loss (\sinreq Loss) is minimized across the finetuning epochs while the accuracy is maximized. This demonstrates a validity for the proposed regularization being able to optimize the two objectives simultaneously.
Figure \ref{fig:conv} (c), (d) contrasts the convergence behavior with and without \sinreq for the case of training from scratch for VGG-11.  
As can be seen, at the onset of training, the accuracy in the presence of \sinreq is behind that without \sinreq. This can be explained as a result of optimizing for an extra objective in case of with \sinreq as compared to without.
Shortly thereafter, the regularization effect kicks in and eventually achieves $\sim6\%$  accuracy improvement.

The convergence behavior, however, is primarily controlled by the regularization strengths $(\lambda_{w})$.
%
As briefly mentioned in Section \ref{sec:sinreq}, $\lambda_q  \in  \lbrack 0, \infty) \,$ is a hyperparameter that weights the relative contribution of the proposed regularization objective to the standard accuracy objective.

%
%
%
We reckon that careful setting of $\lambda_{w}$, $\lambda_{\beta}$ across the layers and during the training epochs is essential for optimum results~\cite{DBLP:journals/corr/abs-1809-00095}.
\begin{figure}
\centering
  \includegraphics[width=0.5\textwidth]{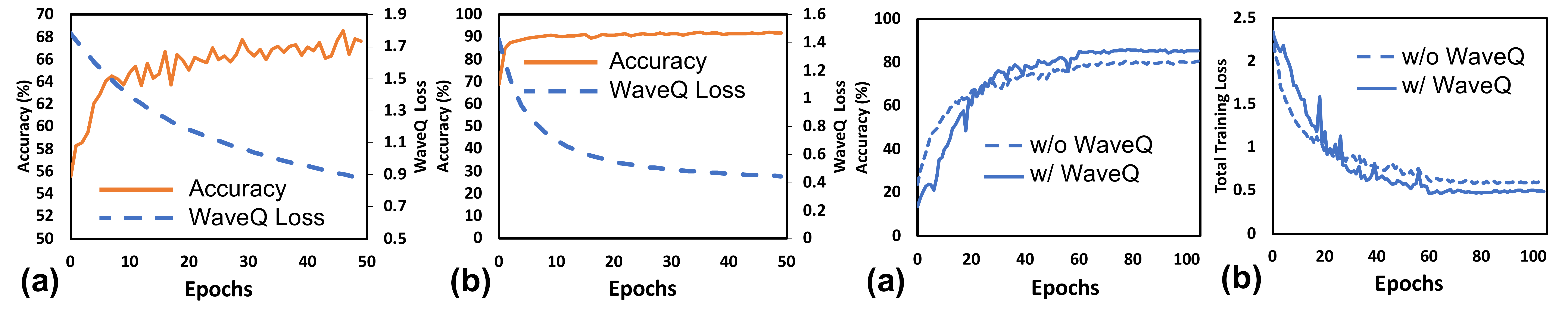}
\vspace{-0.65cm}
  \caption{Convergence behavior: accuracy and \sinreq regularization loss over fine-tuning epochs for (a) CIFAR10, (b) SVHN.
  		 Comparing convergence behavior with and without \sinreq during training from scratch (c) accuracy, (d) training loss. 
		 Network: VGG-11, 2-bit DoReFa quantization}
\vspace{-0.65cm}
  \label{fig:conv}
\end{figure}
\begin{figure*}
  \centering
	\includegraphics[width=0.8\linewidth]{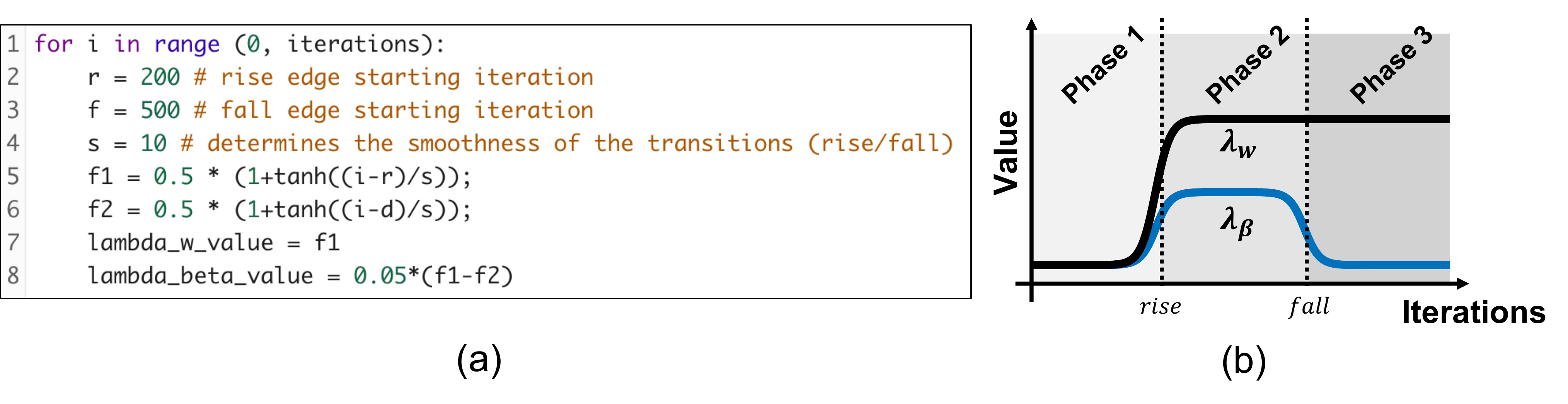}	
	\vspace{-.5em}
	\caption{Math formula for setting $\lambda_w$ and $\lambda_\beta$ during training iterations.}
	\label{fig:math}
	\vspace{-1.5em}
\end{figure*}

\section{Detailed Theoretical Analysis}
\subsection{Motivation}
The results of this section are motivated by the following question. 
\begin{question}
Suppose that a function $F:\mathbb{R}^n\rightarrow [0,\infty)$ has many global minima and that $Q\subset \mathbb{R}^n$ is closed. How do we isolate the global minima of $F$ that are closest to $Q$ without actually computing the full set of global minima of $F$?
\end{question}
Intuitively, we would like to show that if $\epsilon>0$ is very small, then the global minima of the function
\[F(x)+\epsilon d(x,Q)\] are very close to the global minima of $F$ closest to $Q$. To achieve this we will have to introduce first the concept of convergence of sets and then we will show that our intuition is correct by proving that the set of global minima to the above relaxed function converges to a subset of global minima of $F$ closest to $Q$. 

\subsection{Relevant Definitions}
\begin{definition}
If $F:\mathbb{R}^n\rightarrow [0,\infty)$ satisfies $\lim_{|x|\rightarrow\infty}F(x)=+\infty$, we will say that $F$ is coercive.  
\end{definition}
\begin{definition}
For a coercive function $F:\mathbb{R}^n\rightarrow [0,\infty)$ we let $S_{F}=\{x\in\mathbb{R}^n: F(x)=\min_{y\in\mathbb{R}^n}F(y)\}$ be coercive. 
\end{definition}
\begin{lemma}
Assume that $F:\mathbb{R}^n\rightarrow [0,\infty)$ is continuous and coercive. Then $F$ has at least one global minimum. That is, $S_F$ is non-empty. Furthermore, $S_F$ is a compact set. 
\end{lemma}

\begin{definition} 
Let $F,G:\mathbb{R}^n\rightarrow[0,\infty)$ are continuous and assume that $F$ is coercive. Define \[S_{F,G}=\{x\in S_F: G(x)=\inf_{y\in S_F}G(y)\},\] the minima of $F$ which minimize $G$ among the minima of $F$.

\end{definition}

\begin{definition}
Let $Q\subset\mathbb{R}^n$ be a closed set and assume that $x\in\mathbb{R}^n$. Define the distance from $x$ to the set $Q$ to be \[d(x,Q)=\inf_{y\in Q} \|x-y\|.\] Observe that since $Q$ is a closed set we have that $x\in Q$ if and only if $d(x,Q)=0$ and otherwise $d(x,Q)>0$. 
\end{definition}

\begin{definition}
Let $A,B\subset\mathbb{R}^n$ be compact sets. We define the Hausdorff distance between $A$ and $B$ by \[d_{H}(A,B)=\max\{\sup_{x\in A} d(x,A),\sup_{y\in B}d(y,B)\}.\] Observe that $d_H(A,B)=0$ if and only if $A=B$. 
\end{definition}

\begin{definition}
Let $\{S_\delta\}_{\delta>0}$ be a family of compact subsets of $\mathbb{R}^n$. We say that $\lim_{\delta\rightarrow 0} S_\delta=S_*$ if \[\lim_{\delta\rightarrow 0}d_H(S_\delta,S_*)=0.\]
\end{definition}

\begin{lemma}\label{ConvergenceLemma}
Let $S_\delta$ be a family of compact subsets of $\mathbb{R}^n$, then $\lim_{\delta\rightarrow 0} S_\delta =S_*$ if and only if the following two conditions hold.
\begin{enumerate}
\item If $x_\delta\in S_\delta$ converges to $x$, then $x\in S_*$
\item For every $x\in S_*$, there exists a family $x_\delta\in S_\delta$ with $x_\delta\rightarrow x$.  
\end{enumerate}
\end{lemma}

The lemma is just an exercise in the definition. 

\subsection{Statement of the Theorem}

\begin{theorem}\label{SetOfMinimaConverge2}
Let $F,G:\mathbb{R}^n\rightarrow[0,\infty)$ are continuous and assume that $F$ is coercive. Consider the sets $S_{F+\delta G}$, the set of points at which $F+\delta G$ is globally minimum. The following are true:
\begin{enumerate}
\item If $\delta_n\rightarrow 0$ and $S_{F+\delta_n G}\rightarrow S_*$, then \[S_*\subset S_{F,G}\]. 

\item If $\delta_n\rightarrow 0$ then there is a subsequence $\delta_{n_k}\rightarrow 0$ and a non-empty set $S_*\subset S_{F,G}$ so that $S_{F+\delta_{n_k}G}\rightarrow S_*.$
\end{enumerate}

\end{theorem}

\begin{proof}
The second statement follows from the standard theory of Hausdorff distance on compact metric spaces and the first statement. For the first statement, assume that $S_{F+\delta_n G}\rightarrow S_*$. We wish to show that $S_{*}\subset S_{F,G}$. Assume that $x_{n}$ is a sequence of global minima of $F+\delta_{n}G$ converging to $x_*$. It suffices to show that $x_*\in S_{F,G}$. First let us observe that $x_*\in S_F$. Indeed, let \[\lambda=\inf_{x\in \mathbb{R}^n}F(x)\] and assume that $x\in S_F$. Then, 
\[\lambda\leq F(x_{n})\leq (F+\delta_n G)(x_{n})\leq (F+\delta_n G)(x)=\lambda+\delta_n G(x)\rightarrow \lambda.\] Thus, since $F$ is continuous and $x_n\rightarrow x_*$ we have that $F(x_*)=\lambda$ which implies $x_*\in S_{F}$. Next, define \[\mu=\inf_{x\in S_F} G(x).\] Let $\hat x\in S_{F,G}$ so that $G(\hat x)=\mu$. Now observe that, by the minimality of $x_n$ we have that \[\lambda+\delta_n\mu=(F+\delta_n G)(\hat x)\geq (F+\delta_n G)(x_n)\geq \lambda+\delta_n G(x_n)\] Thus, 
\[G(x_n)\leq \mu\] for all $n$. Since $G$ is continuous and $x_n\rightarrow x_*$ we have that $G(x_*)\leq \mu$ which implies that $G(x_*)=\mu$ since $x_*\in S_F$. Thus, $x_*\in S_{F,G}$. 
\end{proof}


\end{document}